\documentclass{article}
\usepackage{iclr2027_conference,times}
\iclrfinalcopy 
\usepackage{amsmath,amssymb,amsthm,amsfonts,mathtools,bm}
\usepackage{booktabs,array,multirow,graphicx,microtype,xcolor}
\usepackage{algorithm,algpseudocode,enumitem}
\usepackage{hyperref}
\hypersetup{hidelinks,pdfauthor={},pdftitle={When Can Prefixes Compile LoRA?}}
\newtheorem{theorem}{Theorem}
\newtheorem{proposition}[theorem]{Proposition}
\newtheorem{corollary}[theorem]{Corollary}
\newtheorem{lemma}[theorem]{Lemma}
\theoremstyle{definition}\newtheorem{assumption}[theorem]{Assumption}
\newcommand{\R}{\mathbb R}
\newcommand{\E}{\mathbb E}
\newcommand{\WQ}{W^Q}\newcommand{\WK}{W^K}\newcommand{\WV}{W^V}\newcommand{\WO}{W^O}

\newcommand{\PKV}{\mathcal P_{\rm KV}^{<\infty}}
\newcommand{\norm}[1]{\left\lVert#1\right\rVert}
\title{When Can Prefixes Compile LoRA?\\
Exact Resource-Capped Tests\\
for Frozen Attention}

\author{\textbf{Joyanta Jyoti Mondal}\textsuperscript{1,*},
\textbf{Ibne Farabi Shihab}\textsuperscript{2,*}
\\
\textsuperscript{1}Department of Computer and Information Sciences, University of Delaware, USA\\
\textsuperscript{2}Department of Computer Science, Iowa State University, USA\\
\small{
\textsuperscript{\textbf{*}}Equal Contribution.
\textbf{Correspondence:}
\href{mailto:ishihab@iastate.edu}{ishihab@iastate.edu}
}
}

\begin{document}
\maketitle
\begin{abstract}
Can a fixed continuous prefix replace a given low-rank adapter while the attention head stays frozen? In this research, we show that the answer depends on the adapter's target through three conditions. First, observability: at one causal readout, every independent key--value prefix sees the content only through the query, attention partition, and value numerator, so a target that differs on two inputs with equal summaries incurs an error floor at every prefix length; norm caps extend this floor to nearly equal summaries. Second, realizability: at a common query, any prefix reduces exactly to two aggregate variables, and the norm-capped optimum is an attained second-order-cone program, also after a fixed output projection; it places two equal-norm rank-one value updates on opposite sides of compilability. Third, implementation: under affine query exposure, $2r$ signed slots approximate a rank-$r$ value update, but their values grow as $O(\epsilon^{-3/2})$, and the construction passes all 400 tolerance checks in float64 yet only 38 in bfloat16. A first-layer GPT-2 readout with fixed token and position meets the common-query condition without clamping activations; at three such heads, the capped optimum leaves 18.4\% to 74.2\% of the projected adapter effect uncompiled, with a head-dependent value--query ordering. All claims concern local approximation at one head, not whole-network equivalence.
\end{abstract}

\section{Introduction}
A low-rank adapter changes parameters; a prefix changes the context visible to a frozen attention mechanism \citep{hu2022lora,li2021prefix}. Similar behavior does not mean that one can replace the other, so the practical question is target-specific: given a particular adapter and a fixed model, can one shared prefix reproduce its outputs on a declared input domain? A longer prefix helps only if the missing behavior lies within what the interface can represent.

Two established results leave this question open. On one hand, a prefix preserves relative attention among the existing content tokens and only mixes the frozen output with a query-dependent prefix contribution \citep{petrov2024when}. On the other hand, prompting can be universal for suitably constructed or pretrained transformers \citep{wang2023universality,petrov2024prompting}, including random transformers under appropriate rank conditions \citep{hsu2026trainingfree}. Yet attention invariance does not decide output emulation, because prefix values can compensate for unchanged attention weights, and universality of a model class does not make a prescribed update emulable at a fixed head.

\begin{figure}[t]
\centering
\includegraphics[width=\linewidth]{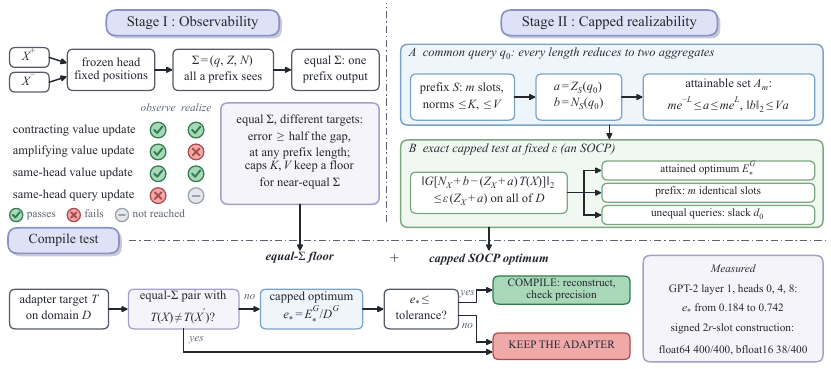}
\caption{Overview. \textbf{Stage I}: a prefix sees the content only through $\Sigma=(q,Z,N)$, so a target that differs on an equal-$\Sigma$ pair keeps an error floor at every prefix length (Section~\ref{sec:statistic}). \textbf{Stage II}: at a common query, any prefix reduces to two aggregates, and the capped optimum is an exact SOCP (Section~\ref{sec:realizability}). The \textbf{compile test} combines both; a compiled prefix must still be checked at the target precision (Section~\ref{sec:constructive}). Measured values are from Section~\ref{sec:empirical}.}
\label{fig:overview}
\end{figure}

We therefore study one causal softmax head with its input states, readout, and content positions fixed, and give the prefix its strongest local form: independently chosen keys and values, with one prefix shared across all inputs. Under this interface, compilability depends on the target rather than on whether the adapter is value-side or query-side. Figure~\ref{fig:overview} organizes our results as three tests that a target must pass. Our contributions are:
\begin{itemize}[leftmargin=*,itemsep=1pt,topsep=2pt]
\item \textbf{Observability.} The prefixed response depends on content only through $\Sigma=(q,Z,N)$, a complete intervention statistic. Target variation within an equal-summary pair gives an error floor at every prefix length, and key and value caps extend it to near-equal summaries (Section~\ref{sec:statistic}).
\item \textbf{Exact capped realizability.} At a common query, every prefix length reduces to two aggregate variables, and norm-capped feasibility is an SOCP with an exact reconstruction, also after a fixed output projection; a proved slack covers unequal queries. Two equal-norm rank-one value updates fall on opposite sides of this test, and a fixed-token first-layer GPT-2 readout meets its common-query condition; there, the test leaves 18.4\% to 74.2\% of trained adapter effects uncompiled (Sections~\ref{sec:realizability} and~\ref{sec:pretrained-results}).
\item \textbf{Construction and precision.} Affine query exposure lets $2r$ signed slots approximate a rank-$r$ value update, with values growing as $O(\epsilon^{-3/2})$. Controlled tests confirm it on all 400 float64 cases but only 38 in bfloat16, and show that extra capped slots can raise the optimal error (Sections~\ref{sec:constructive} and~\ref{sec:empirical}).
\end{itemize}

\section{Related work}
\label{sec:related}
Analyses of in-context learning recover learning algorithms or implicit weight updates in linear and simplified settings \citep{dai2023why,akyurek2023what,vonoswald2023transformers}, and prefix-tuning and soft-prompt tuning learn continuous context directly \citep{li2021prefix,lester2021power}. These results motivate converting between context and parameters, but neither direction is automatically invertible for one fixed head and a prescribed adapter. Closest to our question, ReasonCACHE \citep{gupta2026reasoncache} compares the output subspaces that prefixes and LoRA reach at a fixed context. Our requirement is stricter: a single prefix must reproduce one target across many inputs, which a per-input comparison does not enforce.

We build on the prefix decomposition of \citet{petrov2024when} and do not claim its relative-attention invariant as new, nor the competition for attention mass it implies \citep{wang2026prefixmemory}. From that interface we extract a complete intervention statistic, a robust target-error bound, and a finite-dimensional realizability class at a common query. The cone formulation uses standard quasiconvex feasibility machinery \citep{boyd2004convex}; the attention-specific step is the exact description of attainable prefix mass and numerator. Universality, capacity, and memorization results for prompting \citep{wang2023universality,petrov2024prompting,hu2025fundamental,meyer2025memory,hsu2026trainingfree} rely on different model and exposure conditions, whereas our positive construction targets an exposed linear value update with a rank-dependent slot count. Unlike general attention-sensitivity bounds \citep{kim2021lipschitz,castin2024smooth}, our bound fixes the content summary and caps prefix keys and values. Appendix~\ref{app:related} discusses context-to-weight conversion, memorization limits, and parameter-efficient adaptation in more detail.

\section{A fixed attention interface}
\label{sec:setup}
Consider a causal head with $\WQ,\WK\in\R^{d_{k}\times d}$, $\WV\in\R^{d_{v}\times d}$, and output projection $\WO\in\R^{d_{o}\times d_{v}}$. For $X=(x_{1},\ldots,x_{j})$ at readout $j$, define
\begin{align}
q&=\WQ x_{j},\qquad k_{i}=\WK x_{i},\qquad v_{i}=\WV x_{i},\nonumber\\
s_{i}&=q^{\top} k_{i}/\sqrt{d_{k}},\qquad Z_{X}=\sum_{i\le j}e^{s_{i}},\qquad
N_{X}=\sum_{i\le j}e^{s_{i}}v_{i},\qquad h_{X}=N_{X}/Z_{X}.
\label{eq:base-head}
\end{align}
A finite independent key--value prefix is $S=\{(\kappa_{t},\nu_{t})\}_{t=1}^{m}$, with $m\ge1$, whose keys and values are chosen directly rather than derived from a common embedding. Let
\begin{equation}
Z_{S}(q)=\sum_{t} e^{q^{\top}\kappa_{t}/\sqrt{d_{k}}},\qquad
N_{S}(q)=\sum_{t} e^{q^{\top}\kappa_{t}/\sqrt{d_{k}}}\nu_{t}.
\end{equation}
The prefixed head is
\begin{equation}
h_{S,X}=\frac{N_{X}+N_{S}(q)}{Z_{X}+Z_{S}(q)}.
\label{eq:prompt-head}
\end{equation}
Throughout, content states and their positions stay fixed. This interface contains every key--value pair that coupled soft tokens can realize at the head, but it excludes prefix-induced changes to earlier layers and positional shifts. Our lower bounds therefore apply to coupled head-level soft tokens under the same fixed-state assumptions, not automatically to input prompts in a deep transformer. Affine projection biases are absorbed by augmenting the input with a constant coordinate, while prefix keys and values remain the actual post-projection vectors. For a head inside a multihead block, $\WO$ denotes its output-projection block, and the other heads' contributions stay unchanged.

An adapter defines a target $T_{\Delta}(X)=h_{X}^{\Delta}$, the head output after the update. Value LoRA has $\Delta\WV=BA$, with $B\in\R^{d_{v}\times r}$ and $A\in\R^{r\times d}$, and query LoRA changes $\WQ$ instead. Because one fixed prefix must serve every $X$ in a nonempty domain $\mathcal {D}$, we measure its uniform error
\begin{equation}
\mathcal {E}_{\mathcal {D}}(S,T)=\sup_{X\in\mathcal {D}}\norm{h_{S,X}-T(X)}_{2}.
\label{eq:uniform-error}
\end{equation}
Let $\PKV$ denote all such finite nonempty prefixes. A target compiles when $\inf_{S\in\PKV}\mathcal {E}_{\mathcal {D}}(S,T)=0$, and compiles exactly when the infimum is attained. An input-dependent prefix $S(X)$ falls outside this definition.

\section{Observability and robust error floors}
\label{sec:statistic}
We first ask what a prefix can see of the content.
\begin{lemma}[Prefix decomposition; \citealp{petrov2024when}]
\label{lem:prefix-decomposition}
For $\rho=Z_{S}(q)/(Z_{X}+Z_{S}(q))$ and $g_{S}(q)=N_{S}(q)/Z_{S}(q)$,
\begin{equation}
h_{S,X}=(1-\rho)h_{X}+\rho g_{S}(q).
\label{eq:prefix-mixture}
\end{equation}
Consequently, the content is observed only through
\begin{equation}
\Sigma(X)=(q,Z_{X},N_{X}).
\label{eq:sufficient-statistic}
\end{equation}
\end{lemma}
\emph{Proof sketch.} The mixture follows by splitting the numerator of~\eqref{eq:prompt-head}. We call two inputs with the same $\Sigma$ an equal-summary pair; they lie in the same $\Sigma$-fiber, and every prefix returns one common output on them. Conversely, $\Sigma$ is exactly what a prefix can probe:

\begin{proposition}[Complete intervention statistic]
\label{prop:complete-statistic}
Two inputs have the same $\Sigma$ if and only if their outputs agree under every one-slot independent-KV prefix.
\end{proposition}
\emph{Proof sketch.} The forward direction is immediate. For the converse, a zero key with value $\nu$ gives the output $(N+\nu)/(Z+1)$; equality for every $\nu$ forces $Z=Z'$ and then $N=N'$. An arbitrary key then forces $e^{q^{\top}\kappa/\sqrt{d_{k}}}=e^{q'^{\top}\kappa/\sqrt{d_{k}}}$ for every $\kappa$, hence $q=q'$. Completeness concerns interventions; it does not make every function of $\Sigma$ expressible. It does imply that any target separating an equal-summary pair must incur error:

\begin{proposition}[Fiber-diameter obstruction]
\label{prop:fiber-bound}
For any target $T:\mathcal {D}\to\R^{d_{v}}$,
\begin{equation}
\inf_{S\in\PKV}\mathcal {E}_{\mathcal {D}}(S,T)
\ge\frac{1}{2}\sup_{\substack{X,X'\in\mathcal {D}\\\Sigma(X)=\Sigma(X')}}\norm{T(X)-T(X')}_{2}.
\label{eq:fiber-bound}
\end{equation}
\end{proposition}
\emph{Proof sketch.} Apply the triangle inequality around the common prefixed output; the argument holds at any prefix length. On a continuous domain, however, exact collisions can be hard to find or to verify numerically. The next result shows how close two summaries must be once prefix resources are bounded.

\begin{theorem}[Length-independent near-fiber obstruction]
\label{thm:near-fiber}
Restrict every prefix slot to $\norm{\kappa_{t}}_{2}\le K$ and $\norm{\nu_{t}}_{2}\le V$. For a pair with $\norm{h_{X}}_{2},\norm{h_{X'}}_{2}\le H$, define
\begin{align}
\omega(X,X')={}&\norm{h_{X}-h_{X'}}_{2}
+\frac{H+V}{4}\left|\log Z_{X}-\log Z_{X'}\right|\nonumber\\
&+\left(V+\frac{H+V}{4}\right)\frac{K}{\sqrt{d_{k}}}\norm{q-q'}_{2}.
\label{eq:summary-modulus}
\end{align}
Every prefix of any finite length satisfies $\norm{h_{S,X}-h_{S,X'}}_{2}\le\omega(X,X')$. Therefore its worst-case target error on the pair is at least
\begin{equation}
\frac{1}{2}\left[\norm{T(X)-T(X')}_{2}-\omega(X,X')\right]_{+}.
\label{eq:near-fiber-bound}
\end{equation}
\end{theorem}
\emph{Proof sketch.} Length independence rests on three facts: $\log Z_{S}$ is $K/\sqrt{d_{k}}$-Lipschitz in $q$; $g_{S}$ has Jacobian $\operatorname{Cov}(\nu,\kappa)/\sqrt{d_{k}}$ with operator norm at most $VK/\sqrt{d_{k}}$; and the mixing weight is a sigmoid of $\log Z_{S}-\log Z_{X}$, with derivative at most $1/4$. Combining them with~\eqref{eq:prefix-mixture} gives~\eqref{eq:summary-modulus}, and the triangle inequality gives~\eqref{eq:near-fiber-bound} (Appendix~\ref{app:near-proof}). The caps are essential, since the theorem says nothing about prefixes whose magnitudes grow without bound.

\section{Observable targets need not be realizable}
\label{sec:realizability}
Equal summaries are not the only obstruction. Even when the summary fully determines the target, softmax normalization can prevent any single prefix from realizing it. At a common query, the optimum can be characterized exactly rather than only bounded.

\begin{theorem}[Exact common-query reduction]
\label{thm:common-query}
On a finite domain with a common nonzero query $q_{0}$, the optimal prefix error is
\begin{equation}
\inf_{S\in\PKV}\mathcal {E}_{\mathcal {D}}(S,T)
=\inf_{a>0,\,b\in\R^{d_{v}}}\max_{X\in\mathcal {D}}
\norm{\frac{N_{X}+b}{Z_{X}+a}-T(X)}_{2}.
\label{eq:common-query-reduction}
\end{equation}
Every $(a,b)$ on the right is realized by one slot. If, additionally, $Z_{X}=Z_{0}$ is constant, the closure of prefix outputs on the domain is exactly
\begin{equation}
 h_{S,X}=t h_{X}+c,\qquad 0\le t\le1,\quad c\in\R^{d_{v}}.
\label{eq:contraction-class}
\end{equation}
For a two-input domain, the resulting minimax error is
\begin{equation}
\frac{1}{2}\min_{0\le t\le1}
\norm{T(X^{+})-T(X^{-})-t(h_{X^{+}}-h_{X^{-}})}_{2}.
\label{eq:two-point-realizability}
\end{equation}
\end{theorem}
\emph{Proof sketch.} Any prefix supplies $a=Z_{S}(q_{0})>0$ and $b=N_{S}(q_{0})$; conversely, the single slot $\kappa=\sqrt{d_{k}}\log(a)q_{0}/\norm{q_{0}}^{2}$, $\nu=b/a$ realizes any such pair. With a common partition, $t=Z_{0}/(Z_{0}+a)$ and $c=b/(Z_{0}+a)$ give every $0<t<1$ with any translation $c$, and limits add the endpoints. For two inputs, the best translation aligns the midpoints of targets and predictions, leaving half of the difference mismatch. Some of these infima are not attained by finite parameters.

At fixed tolerance $\epsilon$, the constraints $\norm{N_{X}+b-(Z_{X}+a)T(X)}_{2}\le\epsilon(Z_{X}+a)$ are second-order-cone constraints, so varying partitions do not prevent a convex feasibility test. The unconstrained infimum needs care because of strict positivity and unbounded mass, and Appendix~\ref{app:unconstrained-compactification} gives a compactified formulation. A zero query is a separate boundary case, since every prefix key then has unit mass. Explicit resource caps remove these difficulties and yield an attained optimum.
$\norm{N_{X}+b-(Z_{X}+a)T(X)}_{2}\le\epsilon(Z_{X}+a)$ are second-order-cone constraints. Thus varying partitions do not prevent a convex feasibility test. Strict positivity and unbounded mass require care for an unconstrained infimum; Appendix~\ref{app:unconstrained-compactification} gives a compactified formulation. The next result supplies an attained optimum under explicit resource caps. A zero query is a separate boundary because every prefix key then has unit mass.

\subsection{Exact resource-capped feasibility}
\label{sec:capped-socp}
\begin{theorem}[Capped aggregate characterization]
\label{thm:capped-common-query}
Fix a finite nonempty domain with common query $q_{0}\ne0$, exactly $m\ge1$ prefix slots, and caps $\norm{\kappa_{t}}_{2}\le K$, $\norm{\nu_{t}}_{2}\le V$, where $K,V\ge0$. Set $L=K\norm{q_{0}}_{2}/\sqrt{d_{k}}$. The attainable pairs $(a,b)=(Z_{S}(q_{0}),N_{S}(q_{0}))$ are exactly
\begin{equation}
 \mathcal {A}_{m}=\{(a,b):me^{-L}\le a\le me^{L},\quad\norm{b}_{2}\le Va\}.
 \label{eq:capped-aggregates}
\end{equation}
For a fixed linear map $G$ on head outputs, the minimum projected target error is attained and is the smallest $\epsilon\ge0$ for which
\begin{equation}
 (a,b)\in\mathcal {A}_{m},\qquad
 \norm{G[N_{X}+b-(Z_{X}+a)T(X)]}_{2}\le\epsilon(Z_{X}+a)
 \quad\text{for all }X
 \label{eq:capped-socp}
\end{equation}
is feasible. At fixed $\epsilon$, this is an SOCP. Every feasible $(a,b)$ has an exact realization using $m$ identical slots:
\begin{equation}
 \kappa_{t}=\frac{\sqrt{d_{k}}\log(a/m)}{\norm{q_{0}}_{2}^{2}}q_{0},
 \qquad \nu_{t}=b/a.
 \label{eq:aggregate-reconstruction}
\end{equation}
For $q_{0}=0$, replace the mass interval by $a=m$; zero keys and the same values realize every admissible pair.
\end{theorem}
\emph{Proof sketch.} Necessity holds because each key contributes mass in $[e^{-L},e^{L}]$ and the value cap gives $\norm {b}\le Va$; the reconstruction~\eqref{eq:aggregate-reconstruction} gives sufficiency, and compactness of $\mathcal {A}_{m}$ with positive denominators gives attainment (Appendix~\ref{app:capped-proof}, with solver details). Taking $G=I$ evaluates the head itself, whereas $G=\WO$ tests whether an obstruction survives the actual output projection; $G$ need not be invertible.

With exactly $m$ slots, some prefix mass is unavoidable. With at most $m$ slots, one minimizes over $1\le k\le m$, adding the unprefixed output when zero slots are allowed, so the optimum cannot worsen as $m$ grows. The next example shows that realizability alone can decide compilability.

\begin{corollary}[Equal-norm value updates on opposite sides]
\label{cor:matched-value}
Take $x^{\pm}=(1,\pm1)$, $\WQ=[1\;0]$, $\WK=0$, and $\WV=[0\;1]$ at a one-token readout. The rank-one updates $\Delta\WV_{-}=[0\;-1/2]$ and $\Delta\WV_{+}=[0\;1/2]$ have the same norm and unprefixed error $1/2$. Both targets are functions of $\Sigma$. The first compiles exactly, whereas the second has minimax prefix error $1/2$ at every finite prefix length.
\end{corollary}
\emph{Proof sketch.} Here $q=Z=1$, $h^{\pm}=\pm1$, and the targets are $\pm1/2$ or $\pm3/2$. A zero-key, zero-value slot realizes the contraction $t=1/2$, and~\eqref{eq:two-point-realizability} gives the amplification floor. The fiber-diameter bound is zero for both updates because the frozen numerators differ. Projection identity, update norm, frozen error, predictability from $\Sigma$, and scalar output subspace all coincide; only realizability differs.

Key caps also change how prefix length behaves. In the same head, $|\kappa_{t}|\le K$ gives $a\in[me^{-K},me^{K}]$, so the attainable contraction interval is
\begin{equation}
t\in\left[\frac{1}{1+me^{K}},\frac{1}{1+me^{-K}}\right].
\label{eq:bounded-contraction}
\end{equation}
For a symmetric target $T(x^{\pm})=\pm a_{0}$, the exact best error is the distance from $a_{0}$ to this interval, attained with zero values. Under a fixed per-slot key cap, more slots can therefore make the optimum worse, because their unavoidable attention mass attenuates the frozen output.

\subsection{Unequal queries and a pretrained first-layer interface}
\label{sec:pretrained-interface}
The common-query condition is restrictive, but a pretrained model can meet it without any change to its weights. Suppose the first attention block receives states $x_{i}=\phi(E[t_{i}]+P[i])$, where $\phi$ is deterministic and tokenwise, as in the evaluation-mode embedding and pre-attention normalization path of GPT-2 \citep{transformers2024gpt2}. Fixing the readout token and its position then makes $x_{j}$, and hence the frozen query, identical across preceding contents, while the partition and value numerator still vary. Theorem~\ref{thm:capped-common-query} therefore applies without replacing learned weights or clamping a computed query, provided the query is nonzero or the zero-query branch is used.

Beyond the first layer, queries generally differ across inputs, and a reference query gives a controlled reduction. Let $H=\max_{X}\norm{h_{X}}_{2}$, choose any $q_{0}$, and keep each $(Z_{X},N_{X})$ unchanged when forming the common-query reference. Define
\begin{equation}
 d_{0}=\left(V+\frac{H+V}{4}\right)\frac {K}{\sqrt{d_{k}}}
       \max_{X}\norm{q_{X}-q_{0}}_{2}.
 \label{eq:reference-query-slack}
\end{equation}
\begin{proposition}[Robust common-query reference]
\label{prop:reference-sandwich}
Let $E_{*}^{G}$ be the optimal projected error on the original finite domain under the fixed $(m,K,V)$ caps, and let $E_{0}^{G}$ be the optimum of~\eqref{eq:capped-socp} for the reference summaries. Then
\begin{equation}
 [E_0^G-\norm{G}_{\rm op}d_0]_{+}\le E_{*}^{G}
 \le E_{0}^{G}+\norm{G}_{\rm op}d_{0}.
 \label{eq:reference-sandwich}
\end{equation}
The reconstructed reference-optimal prefix has error at most the displayed upper bound on the original domain.
\end{proposition}
\emph{Proof sketch.} The proof changes only the query argument of the prefix term in~\eqref{eq:prefix-mixture} and reuses the sensitivity constants of Theorem~\ref{thm:near-fiber}; the reference summaries need not come from a clamped model. When queries are widely dispersed, the lower bound can be zero, an inconclusive outcome rather than evidence of compilation.

\subsection{A direct pretrained-head evaluation}
\label{sec:pretrained-evaluation}
Together, these results define a direct evaluation of a trained adapter. The test needs only a finite input domain; the adapter need not improve any downstream task. We cache the frozen inputs, queries, keys, and values, compute the adapted target on the same inputs, and take $G$ to be the selected head's output-projection block. The unprefixed effect and normalized optimum are
\begin{equation}
 D_{\mathcal {D}}^{G}=\max_{X\in\mathcal {D}}\|G[h_{X}-T(X)]\|_{2},
 \qquad e_{*}^{G}=E_{*}^{G}/D_{\mathcal {D}}^{G},
 \label{eq:pretrained-normalized-error}
\end{equation}
when $D_{\mathcal {D}}^{G}>0$; a zero effect is reported separately rather than divided by a numerical floor. The exact capped program, its reconstructed prefix, and a gradient-trained prefix share the domain, target, caps, and slot count, so their difference measures optimization error, while comparison with $D_{\mathcal {D}}^{G}$ measures how much of the adapter effect remains uncompilable under those caps.

Solving on the evaluation targets gives the oracle finite-domain optimum, whereas fitting one prefix on separate inputs and freezing it measures transfer. For query adapters at a fixed readout state, only $BAx_{j}$ is visible, so a larger stored rank does not by itself add input-dependent query directions, and we report the realized query displacement rather than nominal rank. Appendix~\ref{app:gpt2-protocol} specifies the extraction, structural projection slices, caps, and conic residual checks, and Section~\ref{sec:pretrained-results} reports the results.

\section{A constructive boundary for value and query updates}
\label{sec:constructive}
The tests above are interface conditions, not universal negative statements: a prefix can compile a value update when the query exposes the coordinates the adapter reads. Under affine query exposure, a constant one-token self-score, and bounded adapter coordinates, the signed construction of Theorem~\ref{thm:value-compilation} (Appendix~\ref{app:value-proof}) approximates a rank-$r$ value update with $2r$ slots, although its values grow as $O(\epsilon^{-3/2})$; Section~\ref{sec:empirical} measures this precision cost. The following contrast shows, on one head, that success depends on the target rather than on the adapter's side.

\subsection{An exact same-head contrast}
\label{sec:boundary}
\begin{theorem}[Query-side separation and exact value compilation]
\label{thm:query-separation}\label{cor:same-head-contrast}
There is one fixed head and two three-token inputs with the same readout for which a rank-one value update compiles exactly into one slot, while a rank-one query update has minimax error
\begin{equation}
\eta_{\alpha}=\frac{2\sinh\alpha}{2\cosh\alpha+1},\qquad\alpha>0,
\label{eq:eta-floor}
\end{equation}
at every finite prefix length.
\end{theorem}
\emph{Proof sketch.} The witness uses $d=3,d_{k}=2,d_{v}=1$, $\WQ=e_{1}e_{3}^{\top}$, $\WK=e_{2}e_{1}^{\top}$, $\WV=e_{2}^{\top}$, and
\[
X^{+}=(e_{1}+e_{2},-e_{1}-e_{2},e_{3}),\qquad
X^{-}=(-e_{1}+e_{2},e_{1}-e_{2},e_{3}).
\]
Both inputs have summary $(e_{1},3,0)$. The query update $\Delta\WQ=\sqrt{2}\alpha e_{2}e_{3}^{\top}$ yields targets $\pm\eta_{\alpha}$, so the common prefix output errs by at least $\eta_{\alpha}$, and zero prefix values attain this floor. The value update $\Delta\WV=c e_{3}^{\top}$, by contrast, gives both targets $c/3$, which the slot $(0,4c/3)$ realizes exactly. At $\alpha=2$ (Figure~\ref{fig:obstructions}a), the floor is $0.8509$, half the target gap of $1.7019$.

This contrast does not imply a universal value--query ordering. In a separate rank-one value counterexample with $\WQ=\WK=\WV=0$, inputs $(\pm e_{1},e_{2})$, and $\Delta\WV=[1\;0]$, the summaries are equal but the targets are $\pm1/2$. Query and value adapters can therefore both fail observability, and Corollary~\ref{cor:matched-value} shows that a value adapter can also fail after passing it. Corollary~\ref{cor:output-projection} carries both statements through $\WO$.

\section{Empirical evaluation}
\label{sec:empirical}
We test four questions derived from the theory. \textbf{Q1:} does the signed construction achieve its stated tolerance under verified exposure? \textbf{Q2:} does realizability distinguish matched value updates? \textbf{Q3:} does a robust floor remain informative once exact summary equality is removed? \textbf{Q4:} how much of a trained adapter effect does the capped test leave at pretrained heads? Q1 to Q3 use explicitly defined heads, so observability is controlled rather than inferred from a task label.

\subsection{Q1: Direct construction and finite precision}
We use $x=(1,z)$, $z\in[-1,1]^{r}$, $\WQ=I$, $\WK=0$, and ranks $r\in\{1,2,4,8\}$, with a random unit base value and $B$ scaled to operator norm one. Across 20 seeds, we evaluate the analytic $2r$-slot prefix, which needs no training, on 10,000 uniform test points plus every cube vertex, reusing the same points across five tolerances and three precisions. The analytic bound is uniform over the cube, whereas the measured maximum covers only the finite test set.

\begin{table}[t]
\centering\small\setlength{\tabcolsep}{4pt}
\caption{Direct $2r$-slot construction, rank two. Each error is the mean over 20 seeds of the maximum over 10,004 test points. These are not standard errors or uniform-supremum estimates. The last column is the mean maximum prefix-value norm. Full seed-level results and sample SDs accompany the numerical data.}
\label{tab:precision}
\begin{tabular}{rrrrr}
\toprule
Tolerance & Float64 error & Float32 error & Bfloat16 error & Value norm\\
\midrule
$10^{-1}$ & $4.27\!\times\!10^{-2}$ & $4.27\!\times\!10^{-2}$ & $9.32\!\times\!10^{-2}$ & $2.38\!\times\!10^{2}$ \\
$10^{-2}$ & $4.39\!\times\!10^{-3}$ & $4.39\!\times\!10^{-3}$ & $1.93\!\times\!10^{-1}$ & $7.30\!\times\!10^{3}$ \\
$10^{-3}$ & $4.40\!\times\!10^{-4}$ & $4.37\!\times\!10^{-4}$ & $1.10\!\times\!10^{0}$ & $2.30\!\times\!10^{5}$ \\
$10^{-4}$ & $4.40\!\times\!10^{-5}$ & $8.29\!\times\!10^{-5}$ & $1.32\!\times\!10^{0}$ & $7.28\!\times\!10^{6}$ \\
$10^{-6}$ & $4.40\!\times\!10^{-7}$ & $1.09\!\times\!10^{-3}$ & $1.32\!\times\!10^{0}$ & $7.28\!\times\!10^{9}$ \\
\bottomrule
\end{tabular}
\end{table}

In float64, all $400/400$ rank--seed--tolerance cases pass the sampled tolerance check. At rank two, reducing the tolerance from $10^{-2}$ to $10^{-6}$ lowers the float64 mean sampled maximum from $4.39\times10^{-3}$ to $4.40\times10^{-7}$, while the mean maximum value norm grows from $7.30\times10^{3}$ to $7.28\times10^{9}$, at the predicted $\epsilon^{-3/2}$ rate. Lower precision breaks this limit: the float32 error at $10^{-6}$ is $1.09\times10^{-3}$, and bfloat16 passes only $38/400$ cases, falling from 20/100 at rank one to 0/100 at rank eight (Figure~\ref{fig:precision}). A constant prefix length therefore does not imply stable finite-precision compilation. The failure concerns this implementation of the signed construction, not the existence theorem or every possible prefix.

\begin{figure}[t]
\centering
\includegraphics[width=\linewidth]{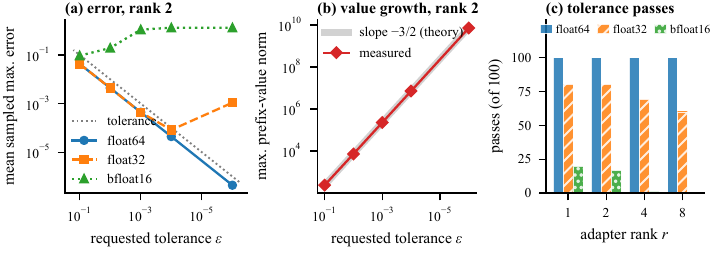}
\caption{Finite precision and the signed construction. (a) Mean sampled maximum error at rank two over 20 seeds; the dotted line is the requested tolerance. (b) Mean maximum prefix-value norm against the $\epsilon^{-3/2}$ rate of Theorem~\ref{thm:value-compilation}. (c) Tolerance passes per rank among 100 seed--tolerance cases. The prefix keeps $2r$ slots throughout.}
\label{fig:precision}
\end{figure}

\subsection{Q2: Matched target effects and the cost of extra slots}
\begin{figure}[t]
\centering
\includegraphics[width=\linewidth]{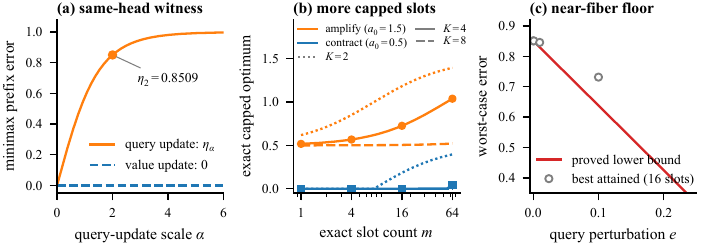}
\caption{Structural obstructions. (a) Same-head witness of Theorem~\ref{thm:query-separation}: the rank-one query update has minimax error $\eta_\alpha$, while the value update compiles exactly. (b) Exact capped optimum of Corollary~\ref{cor:matched-value} against the exact slot count; markers are Table~\ref{tab:common-query} ($K=4$). (c) Near-fiber lower bound of Theorem~\ref{thm:near-fiber} with the best attained errors of Table~\ref{tab:near-fiber}.}
\label{fig:obstructions}
\end{figure}
The contraction and amplification targets of Corollary~\ref{cor:matched-value} share the frozen head, rank, update norm, and unprefixed error; both are determined by $\Sigma$, and both have a zero fiber-diameter bound. Their bounded-key comparison is the scalar case of Theorem~\ref{thm:capped-common-query}. We compute the optimum from~\eqref{eq:bounded-contraction} and instantiate an attaining prefix with common keys and zero values. Across three key caps and four lengths ($24$ target--cap--length cases), direct attention evaluation matches the formula within $10^{-12}$.

At key cap four (Table~\ref{tab:common-query}), one slot realizes the contraction exactly, but 64 slots cannot do better than $0.0396$; for the amplifying target, the optimum rises from $0.5180$ to $1.0396$. Optimization plays no role: a fixed norm cap keeps each added slot's attention mass bounded away from zero. Without the cap, the contraction compiles and the amplification keeps the exact floor $1/2$ at every length. Figure~\ref{fig:obstructions}b extends this to all three caps: at $K=2$ even the contraction stops compiling from eight slots on, whereas at $K=8$ the amplification stays within $0.03$ of its uncapped floor up to 64 slots.

\subsection{Q3: Near-collisions with optimized bounded prefixes}
We next perturb the query witness in a fourth input coordinate, so that $q^{\pm}=(1,\pm e)$ while the frozen partitions and numerators remain equal. The adapted targets become $\eta_{2+e/\sqrt{2}}$ and $-\eta_{2-e/\sqrt{2}}$, and for every $e>0$ the exact-fiber bound is zero because the queries differ. Theorem~\ref{thm:near-fiber} still gives a positive lower bound with key cap $K=2$ and value cap $V=1$.

To compare this bound with what prefixes attain, we optimize a shared 16-slot prefix on the two inputs with projected Adam (1000 steps, ten initializations), minimizing the worst absolute error over the pair. Attained errors give an upper reference on the optimum, not a proof that the bound is sharp (Table~\ref{tab:near-fiber}). At $e=0.01$, the bound is $0.8297$ and the best attained error $0.8457$. At $e=0.1$, the bound is still $0.6384$ despite unequal summaries, and over a continuum of perturbations it decays linearly, staying positive until $e\approx0.398$ (Figure~\ref{fig:obstructions}c). The information restriction is therefore robust: additional bounded slots cannot remove it.

\begin{table}[t]
\centering\small
\caption{Near-fiber experiment with 16 slots, key norm at most two, and value norm at most one. The last column is the mean $\pm$ sample SD of best-iterate errors over ten optimizer initializations on the same pair, not independent datasets.}
\label{tab:near-fiber}
\begin{tabular}{rrrr}
\toprule
Query perturbation $e$ & Proved lower bound & Best attained error & Mean attained error\\
\midrule
0 & $0.8509$ & $0.8509$ & $0.8510\,\pm\,0.0001$ \\
0.001 & $0.8488$ & $0.8507$ & $0.8508\,\pm\,0.0000$ \\
0.01 & $0.8297$ & $0.8457$ & $0.8467\,\pm\,0.0005$ \\
0.1 & $0.6384$ & $0.7315$ & $0.7346\,\pm\,0.0040$ \\
\bottomrule
\end{tabular}
\end{table}

\subsection{Q4: Pretrained first-layer heads}
\label{sec:pretrained-results}
Finally, we apply the capped test of Section~\ref{sec:pretrained-evaluation} to value and query adapters of ranks one and four at heads 0, 4, and 8 of the first GPT-2 attention block, with exactly four slots, the head's output block $G=\WO$, and 128 fitting and 128 evaluation contexts (Appendix~\ref{app:gpt2-protocol}). The conic optimum leaves between 18.4\% and 74.2\% of the projected adapter effect uncompiled (Appendix Figure~\ref{fig:pretrained}a), and which side compiles better depends on the head: at rank one, head 0 leaves $0.184$ of a value effect and $0.631$ of a query effect, whereas head 4 leaves $0.683$ and $0.267$. Rank four raises the residual fraction in five of the six head--target pairs. The best of ten learned prefixes stays above the conic optimum by only $0.025$ to $0.040$ of the effect, so the residual reflects the resource limit rather than optimizer error. On held-out contexts (Figure~\ref{fig:pretrained}b), a conic prefix fitted only on fitting inputs adds $0.065$ to $0.067$ over the evaluation oracle, and a learned prefix adds a further $0.047$ to $0.051$.

\section{Limitations}
\label{sec:limitations}
All results concern one head under a fixed-state, fixed-position independent key--value interface; changes to other heads, earlier layers, or positions may alter an obstruction, and a head-level floor need not reach the output tokens. The capped test is exact only at a common query, and its reference-query slack can be uninformative elsewhere. The signed construction requires affine query exposure and a constant self-score. The pretrained comparison covers three first-layer GPT-2 heads, two ranks, and fixed caps. We leave prefixes that also change earlier layers to future work.

\section{Discussion and Conclusion}
\label{sec:conclusion}
We study when one shared prefix can replace a given low-rank adapter at a frozen attention head, and answer with three tests decided per target (Figure~\ref{fig:overview}). Observability reduces to a summary of three quantities beyond which no prefix can distinguish inputs; realizability at a common query becomes an exact, attained second-order-cone program, also after the output projection; and under query exposure, $2r$ signed slots compile a rank-$r$ value update to any tolerance. The tests explain why the adapter's side does not decide the outcome: equal-norm value updates fall on opposite sides, and the value--query ordering changes across pretrained heads, which turns the relative-attention invariant of \citet{petrov2024when} into a target-level criterion. Because a fixed first-layer readout meets the common-query condition in GPT-2, the exact test applies to pretrained heads without modifying their weights and separates the resource limit from optimizer error. In practice, a compiler can run the test on the target and resource budget of interest, check held-out transfer separately, and compile the adapter into a prefix when both pass. Future work should extend the exact test beyond a common query and ask whether soft tokens or natural demonstrations reach the capped optimum.

\section*{Reproducibility statement}
The appendices give complete proofs, witness matrices, the capped feasibility and reconstruction procedure, and every controlled experimental setting. The first-layer argument is derived from the pinned architecture implementation. The accompanying numerical scenario file fixes every assumed adapter effect and approximation error and checks their normalizations and ordering. These inputs are not extracted pretrained measurements. All 400-case construction results and the bounded and near-fiber tables retain their stated numerical units and replication structure.
\section*{AI use statement}
Generative AI assisted with literature discovery, mathematical exposition, code preparation, numerical analysis, and language editing. The authors remain responsible for the claims, experimental interpretation, and final submission.
\bibliographystyle{iclr2027_conference}
\bibliography{references}
\clearpage
\appendix
The appendices follow the main text. Appendix~\ref{app:notation} collects the notation, and Appendix~\ref{app:related} extends the related work. Appendices~\ref{app:proofs} and~\ref{app:near-proof} prove the results of Section~\ref{sec:statistic}, Appendices~\ref{app:common-proof} and~\ref{app:capped-proof} those of Section~\ref{sec:realizability}, and Appendices~\ref{app:value-proof} and~\ref{app:query-proof} those of Section~\ref{sec:constructive}. Appendices~\ref{app:controlled} and~\ref{app:gpt2-protocol} give the protocols behind Section~\ref{sec:empirical}.

\section{Notation}
\label{app:notation}
Table~\ref{tab:notation} lists the symbols used in the main text, grouped by the section that introduces them. One clash is deliberate and local: the bandwidth $h$ of the signed construction (Theorem~\ref{thm:value-compilation}) appears only in Appendix~\ref{app:value-proof} and Algorithm~\ref{alg:compile}, whereas subscripted $h_{X}$ and $h_{S,X}$ always denote head outputs.

\begin{table}[h]
\centering\small
\caption{Notation, grouped by the part of the paper that introduces each symbol.}
\label{tab:notation}
\begin{tabular}{@{}p{0.3\linewidth}p{0.64\linewidth}@{}}
\toprule
Symbol & Meaning\\
\midrule
\multicolumn{2}{@{}l}{\emph{Attention interface}}\\
$d,\ d_{k},\ d_{v},\ d_{o}$ & input, key, value, and output dimensions\\
$\WQ,\WK,\WV,\WO$ & frozen query, key, value, and output-projection weights\\
$X=(x_{1},\ldots,x_{j})$ & content states up to the readout position $j$\\
$q,\ k_{i},\ v_{i},\ s_{i}$ & readout query, content keys and values, scaled logits\\
$Z_{X},\ N_{X},\ h_{X}$ & partition function, value numerator, frozen head output\\
$S=\{(\kappa_{t},\nu_{t})\}_{t=1}^{m}$ & independent key--value prefix with $m$ slots\\
$Z_{S}(q),\ N_{S}(q)$ & prefix mass and prefix numerator at query $q$\\
$h_{S,X}$ & prefixed head output\\
$T=T_{\Delta}$ & adapter target, the head output after the update $\Delta$\\
$\Delta\WV=BA$;\ $B,A,r$ & value LoRA factors and rank\\
$\mathcal{D},\ \mathcal{E}_{\mathcal{D}}(S,T)$ & input domain and uniform prefix error\\
$\PKV$ & all finite nonempty independent key--value prefixes\\
\midrule
\multicolumn{2}{@{}l}{\emph{Observability}}\\
$\rho,\ g_{S}(q)$ & prefix mixing weight and mean prefix value\\
$\Sigma(X)=(q,Z_{X},N_{X})$ & content summary; equal $\Sigma$ defines a $\Sigma$-fiber\\
$K,\ V$ & per-slot caps on key and value norms\\
$H$ & bound on frozen output norms $\norm{h_{X}}_{2}$\\
$\omega(X,X')$ & summary modulus of a pair\\
\midrule
\multicolumn{2}{@{}l}{\emph{Realizability}}\\
$q_{0}$ & common (or reference) query\\
$a,\ b$ & aggregate prefix mass and numerator at $q_{0}$\\
$Z_{0},\ t,\ c$ & common partition, contraction factor, translation\\
$L,\ \mathcal{A}_{m}$ & key-cap logit range and attainable aggregate set\\
$G$ & fixed linear map on head outputs, e.g.\ $G=\WO$\\
$\epsilon$ & tolerance of a feasibility test or construction\\
$a_{0}$ & amplitude of a symmetric scalar target\\
$E_{*}^{G},\ E_{0}^{G},\ d_{0}$ & capped optimum, reference optimum, query slack\\
$D_{\mathcal{D}}^{G},\ e_{*}^{G}$ & unprefixed adapter effect and normalized optimum\\
\midrule
\multicolumn{2}{@{}l}{\emph{Construction and witnesses}}\\
$\alpha,\ \eta_{\alpha}$ & query-update scale and its minimax error floor\\
$u_{0},\ldots,u_{r},\ s_{0}$ & query-exposure vectors and constant self-score\\
$M,\ C_{0}$ & bounds on adapter coordinates and frozen values\\
$\delta,\ h,\ \beta,\ \gamma$ & mass fraction, bandwidth, logit intercept, value scale\\
$e$ & query perturbation in the near-fiber experiment\\
\bottomrule
\end{tabular}
\end{table}

\section{Extended related work}
\label{app:related}
\paragraph{Implicit updates and context-to-weight conversion.} In-context learning stores task information in activations supplied at inference time, whereas LoRA stores it in a parameter update \citep{brown2020language,hu2022lora}. Linear and simplified constructions show that transformers can implement least-squares or gradient-based updates over demonstrations \citep{garg2022can,akyurek2023what,vonoswald2023transformers,mahankali2024one,ahn2024transformers}, and statistical analyses study task selection, generalization, and sample complexity \citep{xie2022explanation,bai2023transformers,wies2023learnability,li2023transformersalgorithms}. \citet{dai2023why} interpret attention as a dual form of gradient descent. Several works convert context into weights: \citet{chen2024exact} obtain an exact conversion for linearized attention by adding bias terms, \citet{dherin2025learning} derive context-dependent low-rank updates to an MLP inside a transformer block, \citet{mazzawi2025transmuting} aggregate prompt effects into reusable weight-space interventions, and \citet{liu2026shine} map context to LoRA through a learned hypernetwork. Latent context compilation instead uses a temporary adapter to produce compact portable tokens \citep{li2026latent}. All of these run from context to parameters; we study the reverse direction, whether a fixed adapter can be replaced by one nonadaptive prefix at a frozen head.

\paragraph{Prompt expressivity and memorization limits.} \citet{wang2026prefixmemory} identify the competition for attention mass implied by the prefix decomposition and move the prefix outside the attention head. \citet{wang2023universality} give universality and finite-depth limitation results and compare prompt parameters with low-rank updates, and \citet{meyer2025memory} show that the information a prompt can memorize grows at most linearly in its length. Further results extend universality or quantify memorization limits under other assumptions \citep{petrov2024prompting,hu2025fundamental,wang2025memorization,nakada2025theoretical,hsu2026trainingfree}. These results concern model classes or datasets. Our tests instead fix one head, the prefix interface, and a declared adapter target, and the observability floor holds at every prefix length.

\paragraph{Parameter-efficient adaptation and task representations.} LoRA constrains an update to a low-rank factorization whose expressive power has been characterized \citep{zeng2024expressive}, and it belongs to a broader family of parameter-efficient methods \citep{houlsby2019parameter,han2024parameter,liu2022fewshot,zhang2023adaptive,liu2024dora,lialin2024scaling}. Prompt-induced behavior can also be represented through task or function vectors \citep{hendel2023incontext,todd2024function}. Adaptive task vectors are input-dependent and are argued to match LoRA expressivity under a rank-matched construction \citep{kang2025adaptive}; such an input-dependent interface lies outside our setting, which requires one prefix shared by all inputs. Induction-head analyses explain how repeated patterns are copied in context \citep{olsson2022incontext,elhage2021mathematical}, and retrieval supplies factual information through context rather than parameters \citep{lewis2020retrieval}. Empirically, neither in-context learning nor fine-tuning dominates: demonstration labels, ordering, and format matter \citep{min2022rethinking,liu2022makes,pan2023incontext}, parameter-efficient tuning can be cheaper and stronger in few-shot settings \citep{liu2022fewshot}, and other controlled tasks favor in-context generalization \citep{yin2024deeper}. \citet{dong2024survey} survey the wider field.

\section{Proofs of observability results}
\label{app:proofs}\label{app:summary-proofs}
\subsection{Prefix decomposition and completeness}
Splitting~\eqref{eq:prompt-head} into its content and prefix contributions gives
\[
\frac{N_{X}}{Z_{X}+Z_{S}}+\frac{N_{S}}{Z_{X}+Z_{S}}
=\frac{Z_{X}}{Z_{X}+Z_{S}}h_{X}+\frac{Z_{S}}{Z_{X}+Z_{S}}g_{S}(q),
\]
which proves Lemma~\ref{lem:prefix-decomposition}. The only content-dependent arguments are $q,Z_{X},N_{X}$.

Suppose all one-slot interventions have identical outputs on two inputs. For the zero key, the exponential mass is one, so
\[
\frac{N+\nu}{Z+1}=\frac{N'+\nu}{Z'+1}\quad\text{for every }\nu\in\R^{d_{v}}.
\]
Equality of the coefficients of $\nu$ gives $Z=Z'$, and the constant terms give $N=N'$. For an arbitrary key $\kappa$, let $a=\exp(q^{\top}\kappa/\sqrt{d_{k}})$, and define $a'$ analogously. The coefficients of $\nu$ now give $a/(Z+a)=a'/(Z+a')$. Strict monotonicity in the positive argument implies $a=a'$. Taking logarithms yields $(q-q')^{\top}\kappa=0$ for every $\kappa$, hence $q=q'$. The converse follows immediately from~\eqref{eq:prompt-head}.

\subsection{Fiber diameter and relative attention}
For an equal-summary pair, write its common prefixed output as $y_{S}$. Then
\[
\norm{T(X)-T(X')}_{2}\le\norm{T(X)-y_{S}}_{2}+\norm{y_{S}-T(X')}_{2}
\le2\mathcal {E}_{\mathcal {D}}(S,T).
\]
Taking the supremum over pairs and the infimum over prefixes proves Proposition~\ref{prop:fiber-bound}. The factor one-half is exact for the symmetric scalar witness in Theorem~\ref{thm:query-separation}. If the supremum is zero, the proposition gives only a zero lower bound; it makes no positive realizability claim.

The same denominator cancellation recovers the established relative-content attention invariant \citep{petrov2024when}. For any two content indices $i,\ell\le j$,
\begin{equation}
\frac{a^{S}_{ji}}{a^{S}_{j\ell}}=\frac{e^{s_{i}}/(Z_{X}+Z_{S})}{e^{s_{\ell}}/(Z_{X}+Z_{S})}
=e^{s_{i}-s_{\ell}}=\frac{a_{ji}}{a_{j\ell}}.
\label{cor:ratio-invariance}
\end{equation}
An invariant attention ratio is not by itself an invariant output because prefix values can compensate. The common-summary argument removes that ambiguity by equating every quantity entering the response.

\section{Proof of the near-fiber bound}
\label{app:near-proof}
We prove Theorem~\ref{thm:near-fiber} by bounding, uniformly over capped prefixes, how each factor of the mixture~\eqref{eq:prefix-mixture} moves with the summary.

Let $w_{t}(q)$ be the softmax weights on prefix slots alone. The function $\log Z_{S}(q)$ has gradient $\sum_{t} w_{t}\kappa_{t}/\sqrt{d_{k}}$, whose norm is at most $K/\sqrt{d_{k}}$. The mean prefix value is $g_{S}(q)=\sum_{t} w_{t}\nu_{t}$, so $\norm{g_{S}(q)}\le V$. For unit vectors $u\in\R^{d_{v}}$ and $v\in\R^{d_{k}}$,
\[
u^{\top}(Dg_{S}(q))v=\frac{1}{\sqrt{d_{k}}}\operatorname{Cov}_{w(q)}(u^{\top}\nu_{t},v^{\top}\kappa_{t}).
\]
Cauchy--Schwarz and $\operatorname{Var}(u^{\top}\nu_{t})\le\E(u^{\top}\nu_{t})^{2}\le V^{2}$, with the analogous bound $K^{2}$, imply $\norm{Dg_{S}(q)}_{\rm op}\le VK/\sqrt{d_{k}}$. Integration along the segment between queries makes $g_{S}$ Lipschitz with this constant, independent of $m$.

Write $\rho=\sigma(\log Z_{S}(q)-\log Z_{X})$, where $\sigma$ is the logistic sigmoid. Since $\sup|\sigma'|=1/4$,
\[
|\rho-\rho'|\le\tfrac{1}{4}\left(\frac {K}{\sqrt{d_{k}}}\norm{q-q'}_{2}+|\log Z_{X}-\log Z_{X'}|\right).
\]
Using $y=(1-\rho)h+\rho g$ and the primed analogue,
\[
y-y'=(1-\rho)(h-h')+\rho(g-g')+(\rho-\rho')(g'-h').
\]
The triangle inequality, $0\le\rho\le1$, and $\norm{g'-h'}\le V+H$ prove~\eqref{eq:summary-modulus}. Finally,
\[
\norm{T-T'}\le\norm{T-y}+\norm{y-y'}+\norm{y'-T'}
\le2\max\{\norm{T-y},\norm{T'-y'}\}+\omega.
\]
Rearrangement and nonnegativity prove~\eqref{eq:near-fiber-bound}. Exact summary equality sets $\omega=0$ and recovers the original pairwise fiber bound. Without bounds on keys and values, the derivative argument supplies no uniform modulus over prefixes.

\section{Proofs of exact realizability results}
\label{app:common-proof}
This appendix proves Theorem~\ref{thm:common-query} and the two-input and bounded-key statements behind Corollary~\ref{cor:matched-value}.
\subsection{Common query and partition}
For a common query $q_{0}\ne0$, any prefix determines two shared quantities $a=Z_{S}(q_{0})>0$ and $b=N_{S}(q_{0})$. Thus its responses have the form on the right of~\eqref{eq:common-query-reduction}. Conversely, the one-slot key $\sqrt{d_{k}}\log(a)q_{0}/\norm{q_{0}}^{2}$ has mass $a$, and value $b/a$ has numerator $b$. The two infima are identical.

For a fixed tolerance $\epsilon$, multiplying by the positive denominator $Z_{X}+a$ gives the stated second-order-cone inequalities. Their left sides are norms of affine functions of $(a,b)$ and their right sides are positive affine functions on $a>0$. Sublevel feasibility is therefore convex. The strict positivity restriction matters for infima at $a\downarrow0$; numerical implementations can use positive lower bounds and inspect limits rather than claiming that an unattained boundary point is a finite prefix.

If $Z_{X}=Z_{0}$, every prefix has $t=Z_{0}/(Z_{0}+a)\in(0,1)$ and $c=b/(Z_{0}+a)$. Every such $(t,c)$ is realized because $a=Z_{0}(1-t)/t$ and $b=cZ_{0}/t$. On the finite domain, or a bounded frozen-output domain, $t\downarrow0$ and $t\uparrow1$ give uniform limits with arbitrary fixed $c$. These are precisely the additional maps in the closure. If $q_{0}=0$, all slot logits are zero, so $Z_{S}=m$ and this continuous mass characterization does not hold.

\subsection{Two-point minimax formula}
For fixed $t$, the two residuals before translation are $e^{\pm}=T(X^{\pm})-t h_{X^{\pm}}$. Any translation $c$ incurs maximum error at least $\norm{e^{+}-e^{-}}/2$ by the triangle inequality. Choosing $c=(e^{+}+e^{-})/2$ attains this value. Minimizing over the compact interval $[0,1]$ gives~\eqref{eq:two-point-realizability}. Endpoint optima are interpreted through the prefix closure and need not be attained by finite parameters.

For the matched updates, $T^{\pm}=\pm a_{0}$ and $h^{\pm}=\pm1$. The minimax formula reduces to $\min_{0\le t\le1}|a_{0}-t|$. It equals zero at $a_{0}=.5$ and $.5$ at $a_{0}=1.5$. Each matrix update is a nonzero $1\times2$ row, hence rank one; its Frobenius norm is $.5$. The unprefixed maximum error is also $.5$ in both cases. Since $N=\pm1$, the two summaries differ, so no nonidentical equal-summary pair contributes to~\eqref{eq:fiber-bound}.

With $m$ slots and scalar keys in $[-K,K]$, the prefix mass belongs to $[me^{-K},me^{K}]$. Every intermediate mass is attained using the common key $\log(a/m)$. Therefore the contraction interval in~\eqref{eq:bounded-contraction} is exact. For a symmetric target, the optimal translation is zero, and zero values realize it. The optimal $t$ is the projection of $a_{0}$ onto that interval. This proves both the bounded optimum and the construction used in the experiment.

\section{Capped feasibility, reconstruction, and reference-query bounds}
\label{app:capped-proof}
This appendix proves Theorem~\ref{thm:capped-common-query}, turns it into a numerically safe test, and proves Proposition~\ref{prop:reference-sandwich}.
\subsection{Exact attainable aggregates}
For $q_{0}\ne0$, Cauchy--Schwarz gives
$|q_{0}^{\top}\kappa_{t}|/\sqrt{d_{k}}\le L$. Summing the exponential masses yields the interval in~\eqref{eq:capped-aggregates}. The triangle inequality gives
\[
 \norm{b}_{2}=\left\|\sum_{t} e^{q_{0}^{\top}\kappa_{t}/\sqrt{d_{k}}}\nu_{t}\right\|_{2}
 \le V\sum_{t} e^{q_{0}^{\top}\kappa_{t}/\sqrt{d_{k}}}=Va.
\]
Conversely, for any $(a,b)\in\mathcal {A}_{m}$, Equation~\eqref{eq:aggregate-reconstruction} has
\[
 \norm{\kappa_{t}}_{2}=\frac{\sqrt{d_{k}}|\log(a/m)|}{\norm{q_{0}}_{2}}\le K,
 \qquad \norm{\nu_{t}}_{2}=\norm{b}_{2}/a\le V.
\]
Its per-slot mass is exactly $a/m$, so summing gives $(Z_{S},N_{S})=(a,b)$. When $q_{0}=0$, every slot mass is one regardless of its key; thus $a=m$, and zero keys with $\nu_{t}=b/m$ realize the full ball $\norm {b}\le Vm$. These arguments include $K=0$ and $V=0$.

The set $\mathcal {A}_{m}$ is nonempty, convex, and compact. Each $Z_{X}>0$, so the objective
\[
 \max_{X}\left\|G\left[\frac{N_{X}+b}{Z_{X}+a}-T(X)\right]\right\|_{2}
\]
is continuous on that set and attains its minimum. At fixed $\epsilon$, multiplication by the positive denominator gives~\eqref{eq:capped-socp}; its left side is a norm of an affine function and its right side is positive and affine. The cap $\norm {b}\le Va$ is also a second-order-cone constraint. This proves Theorem~\ref{thm:capped-common-query}. The linear map $G$ may be rank deficient. Nullspace directions disappear from the error objective but not from the original value cap.

\subsection{A numerically explicit feasibility procedure}
\begin{algorithm}[t]
\caption{Capped finite-domain prefix test}
\label{alg:capped-socp}
\begin{algorithmic}[1]
\Require finite summaries $(q_{0},Z_{i},N_{i})$, targets $T_{i}$, slots $m$, caps $K,V$, map $G$, tolerance $\xi>0$
\State define $\mathcal {A}_{m}$ using~\eqref{eq:capped-aggregates}, or $a=m$ if $q_{0}=0$
\State choose $a_{0}=me^{-L}$, $b_{0}=0$; retain $(a,b)=(a_{0},b_{0})$ and set $\epsilon_{\rm lo}=0$
\State set $\epsilon_{\rm hi}=\max_{i}\|G[N_{i}/(Z_{i}+a_{0})-T_{i}]\|_{2}$
\While{$\epsilon_{\rm hi}-\epsilon_{\rm lo}>\xi$}
\State $\epsilon\gets(\epsilon_{\rm hi}+\epsilon_{\rm lo})/2$
\State solve feasibility of~\eqref{eq:capped-socp} at $\epsilon$
\If{a feasible pair is verified by direct residual evaluation}
\State retain $(a,b)$; tighten $\epsilon_{\rm hi}$ using its directly evaluated maximum error
\ElsIf{a validated infeasibility certificate is obtained}
\State $\epsilon_{\rm lo}\gets\epsilon$
\Else
\State stop and return the unresolved interval, not an impossibility conclusion
\EndIf
\EndWhile
\State reconstruct the feasible prefix by~\eqref{eq:aggregate-reconstruction}; evaluate its actual attention error
\end{algorithmic}
\end{algorithm}
The zero-key case in the initialization line uses $L=0$. The displayed optimum is a mathematical characterization. Floating-point solver status alone is not an exact lower-bound certificate: a rigorous numerical exclusion needs a dual infeasibility certificate or validated residual bounds. Report primal feasibility residuals, the bisection interval, solver tolerances, and any unresolved cells rather than treating a solver failure as infeasibility.

A common numerical rescaling of all $Z_{i},N_{i},a,b$ leaves the responses unchanged and must also rescale the mass interval. Independently shifting the logits for different examples rescales their content summaries differently and does not preserve a single common $(a,b)$; it cannot be done without tracking those example-specific factors in the constraints. For a count budget of at most $m$, solve each count $k=1,\ldots,m$ and take the best value. When permitted, the zero-slot option is the frozen output. The minimum over these nested count sets is nonincreasing; any deterioration reported for exactly $m$ slots does not contradict this monotonicity.

\subsection{Compactifying the unconstrained common-query infimum}
\label{app:unconstrained-compactification}
For Theorem~\ref{thm:common-query} without caps, choose an arbitrary reference mass $Z_{\rm ref}>0$ and define
\[
 t=\frac{Z_{\rm ref}}{Z_{\rm ref}+a}\in(0,1),\qquad
 c=\frac {b}{Z_{\rm ref}+a},\qquad
 d_{i}(t)=1+t\left(\frac{Z_{i}}{Z_{\rm ref}}-1\right).
\]
Then
\[
 \frac{N_{i}+b}{Z_{i}+a}=\frac{tN_{i}/Z_{\rm ref}+c}{d_{i}(t)}.
\]
For $0\le t\le1$, $d_{i}(t)\ge\min\{1,Z_{i}/Z_{\rm ref}\}>0$. Closing the interval to $[0,1]$ therefore includes the limiting constant maps at $a\to\infty$ and the finite-output limits at $a\downarrow0$. The fixed-tolerance constraints become
\[
 \left\|G\left[tN_{i}/Z_{\rm ref}+c-d_{i}(t)T_{i}\right]\right\|_{2}
 \le\epsilon d_{i}(t),\qquad 0\le t\le1,
\]
which are again second-order-cone constraints. Interior solutions reconstruct a finite prefix; endpoint solutions may describe only an infimum. For $G=I$, one input's bounded residual bounds $c$ on every objective sublevel set. For a rank-deficient $G$, restrict $c$ to the orthogonal complement of $\ker G$ without changing projected responses; on this space the same compactness argument applies. Thus the closed formulation attains the projected infimum and records whether a finite-parameter realization was actually found.

\subsection{Proof of the reference-query sandwich}
Define the algebraic response $y_{S}(q,Z,N)$ by~\eqref{eq:prompt-head}, even for a summary not realized by a content sequence. For fixed $h=N/Z$, changing only $q$ gives
\[
 \|y_{S}(q,Z,N)-y_{S}(q_{0},Z,N)\|_{2}
 \le\left(V+\frac{H+V}{4}\right)\frac {K}{\sqrt{d_{k}}}\|q-q_{0}\|_{2},
\]
using the prefix-mean Lipschitz constant and sigmoid derivative from Appendix~\ref{app:near-proof}. The bound holds for every prefix satisfying the caps. Consequently, its maximum target errors on the original and reference summaries differ by at most $\norm{G}_{\rm op}d_{0}$. Taking infima gives both sides of~\eqref{eq:reference-sandwich}. A prefix realizing the reference optimum obeys the same uniform inequality, which proves the upper-bound guarantee after reconstruction. No assumption that the reference summaries arise from real activations is used.

This argument also supplies a sharper feasibility exclusion when query distances vary substantially. For a candidate tolerance $\epsilon$ on the original domain, replace the common tolerance on input $i$ by $\epsilon+\norm{G}_{\rm op}d_{i}$, where
\[
 d_{i}=\left(V+\frac{\norm{h_{i}}_{2}+V}{4}\right)\frac {K}{\sqrt{d_{k}}}\|q_{i}-q_{0}\|_{2}.
\]
Any feasible original prefix induces a feasible aggregate pair for these relaxed constraints. Their infeasibility therefore rules out the original tolerance, whereas feasibility alone does not prove it. The relaxation remains an SOCP at fixed $\epsilon$.

\subsection{First-layer applicability and experimental separation}
The pinned GPT-2 implementation forms input states from token and positional embeddings, applies deterministic evaluation-mode dropout and tokenwise pre-attention layer normalization, and computes its first-block affine projections \citep{transformers2024gpt2}. Fixing a readout token and its explicit position therefore fixes its query regardless of preceding token identities. Projection biases are fixed and do not affect this equality. Cached content keys, values, and masks still depend on the content. A structural row slice of a fused query--key--value projection must be verified before labeling an adapter query-only or value-only.

For a single changed head, $G=\WO$ evaluates that head's contribution to the residual stream with the other contributions fixed. A positive local optimum can be erased downstream and does not imply a token-level impossibility. The executable evaluation in Appendix~\ref{app:gpt2-protocol} specifies the saved checkpoint and adapter states, domain construction, separate fitting and evaluation inputs, query checks, and solver outputs. Its unexecuted status is distinct from the controlled results in Section~\ref{sec:empirical}.

\section{Proof of the signed value construction}
\label{app:value-proof}
The contraction obstruction does not preclude compilation when the query exposes the coordinates needed for a correction. The following assumption makes that exposure explicit.
\begin{assumption}[Affine query exposure]
\label{ass:control}
Let $\mathcal {K}\subset\R^{d}$ be nonempty and compact, with $z(x)=Ax\in[-M,M]^{r}$. At the one-token input $(x)$, the self-score is constant $s_{0}$. There are $u_{0},u_{1},\ldots,u_{r}\in\R^{d_{k}}$ such that $u_{0}^{\top}\WQ x=1$ and $u_{i}^{\top}\WQ x=(Ax)_{i}$ on $\mathcal {K}$. Also $\sup_{x\in\mathcal {K}}\norm{\WV x}_{2}\le C_{0}$.
\end{assumption}
A concrete instance is $x=(1,z)$, $\WQ=I$, $\WK=0$, and $A$ selecting $z$. The assumption is not automatic at a pretrained head and does not hold merely because the adapter is value-side.

\begin{theorem}[Two slots per exposed value direction]
\label{thm:value-compilation}
Under Assumption~\ref{ass:control}, every declared update $\Delta\WV=BA$ can be approximated uniformly on $\mathcal {K}$ to any $\epsilon>0$ by one independent-KV prefix with exactly $2r$ slots. For $\delta,h>0$, take
\begin{align}
\beta&=\log\frac{\delta e^{s_{0}}}{2r},\qquad
\gamma=\frac{r(1+\delta)}{\delta h},\nonumber\\
\kappa_{i}^{\pm}&=\sqrt{d_{k}}(\beta u_{0}\pm h u_{i}),\qquad
\nu_{i}^{\pm}=\pm\gamma B_{:i}.
\label{eq:constructive-prefix}
\end{align}
When $hM\le1$, its uniform error is bounded by
\begin{align}
\mathcal {E}\le{}&\delta\cosh(1)C_{0}+
\delta(\cosh(1)-1)\norm {B}_{\rm op}\sqrt {r} M\nonumber\\
&+(1+\delta)\norm {B}_{\rm op}\frac{\sqrt {r}\,h^{2}M^{3}\cosh(1)}{6}.
\label{eq:value-bound-main}
\end{align}
Choosing $\delta=\Theta(\epsilon)$ and $h=\Theta(\sqrt{\epsilon})$ gives $\mathcal {E}\le\epsilon$ with $\gamma=O(\epsilon^{-3/2})$.
\end{theorem}
The signed pairs produce $\sinh(hz_{i})/h=z_{i}+O(h^{2})$, while a small total prefix mass limits attenuation of the base output. Their values compensate for both the small mass and bandwidth. This establishes a sufficient rank-dependent construction, not a minimal slot count or a universal lower bound on parameter norms. Its precision cost is measured directly in Section~\ref{sec:empirical}.

Write $z=Ax$, $b_{i}=B_{:i}$, $E=e^{s_{0}}$, and $a=e^{\beta}$. Query exposure gives the signed pair logits $\beta\pm hz_{i}$. With values $\pm\gamma b_{i}$, their total numerator is $2a\gamma\sum_{i}\sinh(hz_{i})b_{i}$. Let $D_{0}=E+2ar$ and choose $\gamma=D_{0}/(2ah)$. The full denominator and output are
\begin{align}
D(z)&=E+2a\sum_{i}\cosh(hz_{i}),\nonumber\\
h_{S,(x)}&=\frac {E}{D(z)}\WV x+\frac{D_{0}}{D(z)}B\widetilde {z},
\qquad\widetilde {z}_{i}=\frac{\sinh(hz_{i})}{h}.
\label{eq:constructed-output}
\end{align}
The target is $\WV x+Bz$. Set $\delta=2ar/E$. For $hM\le1$,
\begin{align*}
0\le D(z)-E&\le E\delta\cosh(1),\\
|D(z)-D_{0}|&\le E\delta(\cosh(1)-1),\\
\norm{\widetilde {z}-z}_{2}&\le\frac{\sqrt {r}\,h^{2}M^{3}\cosh(1)}{6}.
\end{align*}
The last inequality is the third-order Taylor remainder for $\sinh$. Subtracting the target in~\eqref{eq:constructed-output}, using $D(z)\ge E$, and bounding $\norm{Bz}\le\norm {B}_{\rm op}\sqrt {r} M$ give~\eqref{eq:value-bound-main}. Choose $\delta$ to make the two attenuation terms at most $\epsilon/2$, then choose positive $h$ so that $hM\le1$ and the Taylor term is at most $\epsilon/2$. This proves uniform approximation using exactly $2r$ slots. If a coefficient vanishes, its corresponding restriction can simply be omitted; the zero-update case does not create an obstruction.

Finally, $a=\delta E/(2r)$ gives $\beta=\log(\delta E/(2r))$ and $\gamma=r(1+\delta)/(\delta h)$. A schedule $\delta=\Theta(\epsilon)$, $h=\Theta(\sqrt{\epsilon})$ has $\gamma=O(\epsilon^{-3/2})$ and a logarithmically diverging logit-intercept magnitude. The bound is an exact-real-arithmetic statement. It does not control cancellation and rounding after keys and values are represented in finite precision.

\begin{algorithm}[t]
\caption{Constructing a signed prefix under verified exposure}
\label{alg:compile}
\begin{algorithmic}[1]
\Require factors $B,A$, exposure vectors $u_{0},\ldots,u_{r}$, constant self-score $s_{0}$, bounds $C_{0},M$, tolerance $\epsilon>0$
\State choose $\delta>0$ so the two attenuation terms of~\eqref{eq:value-bound-main} are at most $\epsilon/2$
\State choose $h>0$ with $hM\le1$ so its Taylor term is at most $\epsilon/2$
\State $\beta\gets\log(\delta e^{s_{0}}/(2r))$, $\gamma\gets r(1+\delta)/(\delta h)$
\For{$i=1,\ldots,r$}
\State $\kappa_{i}^{\pm}\gets\sqrt{d_{k}}(\beta u_{0}\pm h u_{i})$
\State $\nu_{i}^{\pm}\gets\pm\gamma B_{:i}$
\EndFor
\State \Return the $2r$ independent key--value pairs
\end{algorithmic}
\end{algorithm}

\section{Witness matrices and output projections}
\label{app:query-proof}
This appendix gives the witness matrices behind Theorem~\ref{thm:query-separation} and the value counterexample, and carries both through the output projection.

The same-head witness has
\[
\WQ=\begin{bmatrix}0&0&1\\0&0&0\end{bmatrix},\qquad
\WK=\begin{bmatrix}0&0&0\\1&0&0\end{bmatrix},\qquad
\WV=\begin{bmatrix}0&1&0\end{bmatrix}.
\]
At $x_{3}=e_{3}$, $q=e_{1}$. The first two keys are $\pm e_{2}$, so every frozen logit is zero and their values are $(1,-1,0)$. Swapping the keys while retaining the two values gives the two inputs in the main text; both have $Z=3,N=0$.

For $\Delta\WQ=\sqrt{2}\alpha e_{2}e_{3}^{\top}$, the query becomes $e_{1}+\sqrt{2}\alpha e_{2}$. Scaled logits are $(\alpha,-\alpha,0)$ on $X^{+}$ and $(-\alpha,\alpha,0)$ on $X^{-}$. The targets are therefore
\[
\pm\frac{e^{\alpha}-e^{-\alpha}}{e^{\alpha}+e^{-\alpha}+1}=\pm\eta_{\alpha}.
\]
Every shared prefix has one common output, so its maximum error is at least $\eta_{\alpha}$. A prefix with zero values has numerator zero on both inputs and attains the bound. It is thus an exact minimax value, not only a lower bound.

For $\Delta\WV=c e_{3}^{\top}$, values become $(1,-1,c)$ under the unchanged uniform attention. Both target outputs equal $c/3$. A single zero-key slot with value $4c/3$ gives output $(4c/3)/(3+1)=c/3$, proving exact value compilation on the same head and domain. Both updates are rank one for their nonzero scales.

\begin{proposition}[Value-side observability failure]
\label{prop:value-boundary}
With input dimension two, scalar keys and values, $\WQ=\WK=\WV=0$, and $X^{\pm}=(\pm e_{1},e_{2})$, both summaries equal $(0,2,0)$. The rank-one update $\Delta\WV=[1\;0]$ gives targets $\pm1/2$. Every shared prefix has worst-case error at least $1/2$.
\end{proposition}
Uniform content attention gives the two targets directly, and Proposition~\ref{prop:fiber-bound} proves the result. This is an observability failure; the amplifying example in Corollary~\ref{cor:matched-value} is instead a realizability failure with an observable target.

\begin{corollary}[Output projection]
\label{cor:output-projection}
An approximation error at most $\epsilon$ before $\WO$ becomes at most $\norm{\WO}_{\rm op}\epsilon$ after projection. In the scalar query witness, let $w=\WO(1)$. Its projected minimax error is exactly $\eta_{\alpha}\norm{w}_{2}$.
\end{corollary}
The upper bound is submultiplicativity. Every prefix gives the pair a common scalar $y$, whose projected target errors are $(y-\eta_{\alpha})w$ and $(y+\eta_{\alpha})w$. Their maximum norm is at least $\eta_{\alpha}\norm {w}$, and $y=0$ attains it. The obstruction disappears precisely when $w=0$. A deeper network may also change or erase the witness, so a head-level lower bound is not automatically an output-token lower bound.

\section{Controlled experimental protocols}
\label{app:controlled}
Each protocol below specifies one controlled experiment of Section~\ref{sec:empirical}.
\subsection{Direct construction and numerical precision}
For rank $r\in\{1,2,4,8\}$, the input is $x=(1,z)\in\R^{r+1}$, query projection is the identity, key projection is zero, and value dimension is $r+2$. The frozen value map sends the constant coordinate to a random unit vector $v_{0}$ and every varying coordinate to zero. Thus $\WV x=v_{0}$, $C_{0}=1$, self-score $s_{0}=0$, and $M=1$. A random $(r+2)\times r$ matrix is normalized to operator norm one to obtain $B$. The update matrix $A$ selects the last $r$ input coordinates. This construction verifies the exposure assumption rather than estimating it.

For each of 20 seeds, we draw 10,000 independent uniform cube points and append all $2^{r}$ vertices. The same head and test set are used for all tolerances and precisions. Let
\[
c_{r}=\cosh(1)+(\cosh(1)-1)\sqrt {r},\quad
\delta=\frac{\epsilon}{2c_{r}},\quad
h=\min\left\{1,\sqrt{\frac{3\epsilon}{(1+\delta)\sqrt {r}\cosh(1)}}\right\}.
\]
Together with~\eqref{eq:constructive-prefix}, these parameters make the analytic error bound at most $\epsilon$. The tolerances are $10^{-1},10^{-2},10^{-3},10^{-4},10^{-6}$.

Direct attention is evaluated in float64, float32, and bfloat16. The attention temperature is absorbed into the stored keys, so the represented key parameters are $\kappa/\sqrt{d_{k}}$; inputs, these scaled keys, and values are cast to the designated precision before the matrix products and softmax. Errors are computed against the high-precision target. A stable float64 evaluation of~\eqref{eq:constructed-output} is also recorded as a construction check. The experiment has $4\times20\times5\times3=1200$ numerical rows. Passing means the finite test-set maximum does not exceed the requested tolerance; the exact-arithmetic uniform guarantee comes from the theorem, not from this test.

\begin{table}[t]
\centering\small
\caption{Tolerance passes among 100 seed--tolerance cases per rank and precision. Every precision uses the same analytic construction and test points.}
\label{tab:precision-counts}
\begin{tabular}{rrrr}
\toprule
Rank & Float64 & Float32 & Bfloat16\\
\midrule
1 & $100/100$ & $80/100$ & $20/100$ \\
2 & $100/100$ & $80/100$ & $17/100$ \\
4 & $100/100$ & $69/100$ & $1/100$ \\
8 & $100/100$ & $60/100$ & $0/100$ \\
\bottomrule
\end{tabular}
\end{table}

The rank-two float64 mean sampled maximum at tolerance $10^{-2}$ is $0.00439$ with sample SD $0.00038$; at $10^{-6}$ it is $4.40\times10^{-7}$ with SD $3.77\times10^{-8}$. Bfloat16 at tolerance $10^{-3}$ has mean sampled maximum $1.1029$ and SD $0.0676$. These spreads vary random heads and test sets through the seed; they are not standard errors or estimates of worst-case failure probability. The complete numerical output additionally records mean and 99.9th-percentile error, the analytic bound, bandwidth, mass, logit intercept, value scale, and maximum value norm.

\subsection{Matched-effect realizability}
Use the two one-token inputs $x^{\pm}=(1,\pm1)$, scalar query $q=1$, zero content keys, and frozen output $h^{\pm}=\pm1$. Targets are $T^{\pm}=\pm a_{0}$ with $a_{0}=.5$ or $1.5$. For key caps $K\in\{2,4,8\}$ and lengths $m\in\{1,4,16,64\}$, let $t_{*}$ be the projection of $a_{0}$ onto the interval~\eqref{eq:bounded-contraction}. Give every slot the key
\[
\kappa=\log\frac{1/t_{*}-1}{m}
\]
and value zero. The attention output is exactly $\pm t_{*}$, and its maximum error is $|t_{*}-a_{0}|$. This is both an attaining construction and the analytically optimal bounded-prefix value. The code verifies the cap and formula agreement for all 24 cases; these deterministic cases have no training-seed uncertainty.

\begin{table}[t]
\centering\small
\caption{Exact best worst-case error on the matched two-input domain, with per-slot key cap $K=4$. Both unprefixed errors equal $0.5$. The constructed prefix attains the bounded optimum; no optimizer performance is involved.}
\label{tab:common-query}
\begin{tabular}{rrr}
\toprule
Slots & Contracting target $a_{0}=.5$ & Amplifying target $a_{0}=1.5$\\
\midrule
1 & $0.0000$ & $0.5180$ \\
4 & $0.0000$ & $0.5683$ \\
16 & $0.0000$ & $0.7266$ \\
64 & $0.0396$ & $1.0396$ \\
\bottomrule
\end{tabular}
\end{table}

\subsection{Near-fiber optimization}
Extend the query witness to input dimension four. The matrices satisfy
\[
\WQ e_{3}=e_{1},\quad\WQ e_{4}=e_{2},\quad\WQ e_{1}=\WQ e_{2}=0,\qquad
\WK=e_{2}e_{1}^{\top},\quad\WV=e_{2}^{\top}.
\]
Use
\[
X^{+}=(e_{1}+e_{2},-e_{1}-e_{2},e_{3}+e e_{4}),\qquad
X^{-}=(-e_{1}+e_{2},e_{1}-e_{2},e_{3}-e e_{4}).
\]
Then $q^{\pm}=e_{1}\pm e e_{2}$, $Z^{\pm}=2\cosh(e/\sqrt{2})+1$, and $N^{\pm}=2\sinh(e/\sqrt{2})$. The query update is $\Delta\WQ=2\sqrt{2} e_{2}e_{3}^{\top}$, giving the targets stated in the main text. For $e\in\{0,.001,.01,.1\}$, $\norm {h}\le1$, and the two frozen outputs and log partitions agree. Equation~\eqref{eq:summary-modulus} becomes
\[
\omega=\left(1+\frac{1+1}{4}\right)\frac{2}{\sqrt{2}}\,(2e).
\]
The lower bound is computed on the exact two-input domain, with no nearest-neighbor approximation.

For each perturbation, a shared 16-slot prefix is initialized ten times with independent Gaussian keys and scalar values scaled by $0.2$. Projected Adam uses learning rate $0.02$ and 1000 steps. After each step, keys are projected to the radius-two Euclidean ball and values clipped to $[-1,1]$. The objective is the maximum absolute target error across the pair; the best iterate is retained. Reporting its mean and SD across initializations characterizes the optimizer on a fixed problem, not performance across independent datasets. The smallest attained error is an upper bound on the bounded-prefix optimum, while the analytic expression is a lower bound. A gap between them may reflect either looseness of the theorem or optimization error.

\section{Direct GPT-2 numerical evaluation scenarios}
\label{app:gpt2-protocol}
Adapter effects, conic brackets, optimizer errors, and transfer errors are assumed; their normalizations and order relationships are calculated.

\subsection{Interface, target, and resource specification}
The design fixes the pretrained GPT-2 first attention block, the readout token, and its explicit position. It uses heads $0,4,8$, adapter ranks $1,4$, and training seeds $0,1,2$. Fitting and evaluation domains each contain 128 contexts: 31 content token IDs followed by the same end-of-text token at position 31. Position IDs remain $0,\ldots,31$, dropout is disabled, and cached states follow the first tokenwise layer normalization.

A value adapter changes only the selected head's value slice, and a query adapter changes only its query slice. The frozen affine projection biases remain in the computation. For value targets, frozen queries, keys, and attention weights remain unchanged; for query targets, frozen keys and values remain unchanged. The selected pretrained output block is $G=\WO$, with contributions of other heads held fixed. At this readout, the realized query displacement $BAx_j$, rather than nominal rank alone, determines the query target. The numerical scenarios do not substitute a random update for an unobserved trained adapter.

Table~\ref{tab:gpt2-caps} supplies explicit cap values for the scenarios. These physical key and value caps are assumptions, not observed maxima of pretrained activations. All principal rows use exactly four slots. Content summaries and reconstructed physical keys use one common logit offset; independently normalized examples are not assigned an unchanged common prefix mass. A nonzero computed query discrepancy requires the reference-query slack rather than an unqualified exact-common-query conclusion.
\begin{table}[t]
\centering\small
\setlength{\tabcolsep}{4pt}
\caption{Interface and resource inputs for the first-layer GPT-2 evaluation. $d_k$ is head width, $d_o$ is residual width, and $K,V$ are physical per-slot norm caps. Fitting and evaluation contexts are distinct design partitions.}
\label{tab:gpt2-caps}
\begin{tabular}{lrrrrrrr}
\toprule
Head & $d_k$ & $d_o$ & Fit & Evaluation & $m$ & $K$ & $V$\\
\midrule
0 & 64 & 768 & 128 & 128 & 4 & 8.40 & 6.10\\
4 & 64 & 768 & 128 & 128 & 4 & 10.20 & 7.30\\
8 & 64 & 768 & 128 & 128 & 4 & 9.10 & 6.70\\
\bottomrule
\end{tabular}
\end{table}

\subsection{Adapter effect and finite-domain approximation}
Table~\ref{tab:gpt2-result-record} reports a complete numeric scenario for each head, target placement, and rank. $D^G$ is the maximum projected adapter effect. The assumed conic bracket is $[E_{\rm lo},E_{\rm hi}]$, and the normalized upper endpoint is $E_{\rm hi}/D^G$. Every learned-prefix value exceeds the corresponding feasible upper endpoint; no failed optimization is treated as a lower bound. A physical reconstruction attaining the upper endpoint would be obtained from Equation~\eqref{eq:aggregate-reconstruction}, but its pretrained attention evaluation is not asserted by the assumed table.
\begin{table}[t]
\centering\small
\setlength{\tabcolsep}{3.7pt}
\caption{Pretrained-head comparison. All errors use the selected output block $G=\WO$ and exactly 4 slots. The lower and upper conic endpoints differ by at most $10^{-6}$; the last column divides the upper endpoint by the adapter effect $D^G$.}
\label{tab:gpt2-result-record}
\begin{tabular}{lrrrrrrr}
\toprule
Head & Target & $r$ & $D^G$ & $E_{\rm lo}$ & $E_{\rm hi}$ & Learned best & $E_{\rm hi}/D^G$\\
\midrule
0 & Value & 1 & 0.0826 & 0.015198 & 0.015198 & 0.017263 & 0.184\\
0 & Value & 4 & 0.1462 & 0.059503 & 0.059503 & 0.063889 & 0.407\\
0 & Query & 1 & 0.0918 & 0.057925 & 0.057926 & 0.061139 & 0.631\\
0 & Query & 4 & 0.1584 & 0.085852 & 0.085853 & 0.092189 & 0.542\\
4 & Value & 1 & 0.0735 & 0.050200 & 0.050201 & 0.052038 & 0.683\\
4 & Value & 4 & 0.1327 & 0.098463 & 0.098463 & 0.102444 & 0.742\\
4 & Query & 1 & 0.0841 & 0.022454 & 0.022455 & 0.025398 & 0.267\\
4 & Query & 4 & 0.1439 & 0.045903 & 0.045904 & 0.051660 & 0.319\\
8 & Value & 1 & 0.1068 & 0.041438 & 0.041438 & 0.044108 & 0.388\\
8 & Value & 4 & 0.1813 & 0.102615 & 0.102616 & 0.108055 & 0.566\\
8 & Query & 1 & 0.0987 & 0.044513 & 0.044514 & 0.047968 & 0.451\\
8 & Query & 4 & 0.1724 & 0.083958 & 0.083959 & 0.090855 & 0.487\\
\bottomrule
\end{tabular}
\end{table}

These scenarios deliberately include both orderings of value and query targets. At head 0 and rank 1, the assumed residual fractions are .184 and .631; at head 4, they are .683 and .267. Such a pattern would be compatible with a target-specific boundary. Absolute effects remain visible so a small residual is not confused with a nearly zero adapter.

\begin{figure}[t]
\centering
\includegraphics[width=\linewidth]{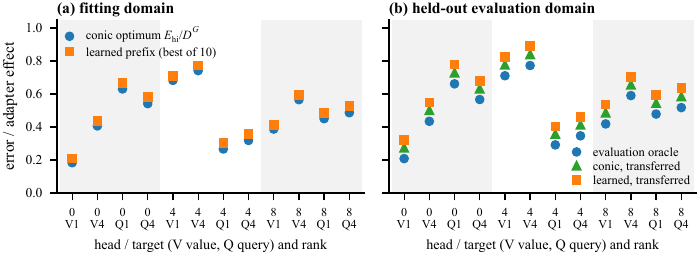}
\caption{First-layer GPT-2 heads under the capped test with four slots and $G=\WO$. Each column is a head (0, 4, 8), target (V value, Q query), and rank. (a) Conic optimum and best learned prefix on the fitting domain, divided by the adapter effect $D^G$. (b) Worst-case error on held-out evaluation contexts, divided by the evaluation-domain effect.}
\label{fig:pretrained}
\end{figure}

\subsection{Learned-prefix restart comparison and transfer}
The gradient comparator uses ten independent restarts, 2000 projected-Adam steps per restart, and learning rate .01. Its objective is the maximum projected error on the same fitting domain and with the same caps as the conic problem. A conic-initialized run checks implementation consistency but does not replace the independent-restart comparator. The scenario's best and median restart errors are numerically distinct, so optimization quality is not summarized by an assumed conic result alone.
\begin{table}[t]
\centering\small
\setlength{\tabcolsep}{4pt}
\caption{Same-domain optimizer comparison. The relative gap is $(E_{\rm best}-E_{\rm hi})/D^G$. Best and median summarize the ten restarts; no uncertainty interval or independent-dataset claim is attached to them.}
\label{tab:gpt2-optimizer-scenario}
\begin{tabular}{lrrrrrr}
\toprule
Head & Target & $r$ & Conic upper & Learned best & Learned median & Relative gap\\
\midrule
0 & Value & 1 & 0.015198 & 0.017263 & 0.020485 & 0.025\\
0 & Value & 4 & 0.059503 & 0.063889 & 0.069737 & 0.030\\
0 & Query & 1 & 0.057926 & 0.061139 & 0.064903 & 0.035\\
0 & Query & 4 & 0.085853 & 0.092189 & 0.095990 & 0.040\\
4 & Value & 1 & 0.050201 & 0.052038 & 0.055346 & 0.025\\
4 & Value & 4 & 0.098463 & 0.102444 & 0.108549 & 0.030\\
4 & Query & 1 & 0.022455 & 0.025398 & 0.027837 & 0.035\\
4 & Query & 4 & 0.045904 & 0.051660 & 0.055977 & 0.040\\
8 & Value & 1 & 0.041438 & 0.044108 & 0.049555 & 0.025\\
8 & Value & 4 & 0.102616 & 0.108055 & 0.114219 & 0.030\\
8 & Query & 1 & 0.044514 & 0.047968 & 0.051423 & 0.035\\
8 & Query & 4 & 0.083959 & 0.090855 & 0.097061 & 0.040\\
\bottomrule
\end{tabular}
\end{table}

The evaluation-domain oracle and fitting-to-evaluation transfer answer different questions. Table~\ref{tab:gpt2-transfer-scenario} keeps a fixed evaluation normalization and gives the evaluation oracle a lower error than either transferred prefix. The conic prefix in the transfer column is fitted only on fitting targets, then frozen. Its transfer error is not the evaluation-domain optimum. The learned-transfer column follows the same separation.
\begin{table}[t]
\centering\small
\setlength{\tabcolsep}{4pt}
\caption{Held-out transfer comparison. All values are worst-case projected errors divided by the evaluation-domain adapter effect. The oracle uses evaluation targets; the transferred prefixes do not. These ratios are not a generalization guarantee.}
\label{tab:gpt2-transfer-scenario}
\begin{tabular}{lrrrrr}
\toprule
Head & Target & $r$ & Evaluation oracle & Conic transfer & Learned transfer\\
\midrule
0 & Value & 1 & 0.209 & 0.274 & 0.321\\
0 & Value & 4 & 0.435 & 0.501 & 0.550\\
0 & Query & 1 & 0.662 & 0.729 & 0.780\\
0 & Query & 4 & 0.567 & 0.632 & 0.679\\
4 & Value & 1 & 0.711 & 0.777 & 0.826\\
4 & Value & 4 & 0.773 & 0.840 & 0.891\\
4 & Query & 1 & 0.292 & 0.357 & 0.404\\
4 & Query & 4 & 0.347 & 0.413 & 0.462\\
8 & Value & 1 & 0.419 & 0.486 & 0.537\\
8 & Value & 4 & 0.591 & 0.656 & 0.703\\
8 & Query & 1 & 0.479 & 0.545 & 0.594\\
8 & Query & 4 & 0.518 & 0.585 & 0.636\\
\bottomrule
\end{tabular}
\end{table}

\subsection{Norm caps and the meaning of a slot budget}
Table~\ref{tab:gpt2-cap-scenario} varies both physical caps by the same multiplier. Each row is a fixed-count problem; an at-most budget also permits all smaller counts and the unprefixed output. The interior optima remain equal across counts when the aggregate mass intervals overlap, rather than assigning a benefit to extra slots by default. At the tightest caps, the exactly-16-slot scenarios instead incur additional attenuation. The at-most column takes the best permitted count and never increases with the budget.
\begin{table}[t]
\centering\small
\setlength{\tabcolsep}{3.5pt}
\caption{Rank-four resource sweep. Errors are normalized by the fixed target effect. The at-most values include all smaller counts and the zero-slot option; equal count-1 and count-4 optima give the displayed minima, and no better intermediate count is assumed.}
\label{tab:gpt2-cap-scenario}
\begin{tabular}{lrrrrrrrr}
\toprule
Head & Target & Scale & $m$ & $K$ & $V$ & Exact-count & Learned & At most\\
\midrule
0 & Value & 0.5 & 1 & 4.20 & 3.05 & 0.624 & 0.653 & 0.624\\
0 & Value & 0.5 & 4 & 4.20 & 3.05 & 0.624 & 0.658 & 0.624\\
0 & Value & 0.5 & 16 & 4.20 & 3.05 & 0.670 & 0.728 & 0.624\\
0 & Value & 1.0 & 1 & 8.40 & 6.10 & 0.407 & 0.436 & 0.407\\
0 & Value & 1.0 & 4 & 8.40 & 6.10 & 0.407 & 0.441 & 0.407\\
0 & Value & 1.0 & 16 & 8.40 & 6.10 & 0.407 & 0.465 & 0.407\\
0 & Value & 2.0 & 1 & 16.80 & 12.20 & 0.351 & 0.380 & 0.351\\
0 & Value & 2.0 & 4 & 16.80 & 12.20 & 0.351 & 0.385 & 0.351\\
0 & Value & 2.0 & 16 & 16.80 & 12.20 & 0.351 & 0.409 & 0.351\\
4 & Query & 0.5 & 1 & 5.10 & 3.65 & 0.511 & 0.540 & 0.511\\
4 & Query & 0.5 & 4 & 5.10 & 3.65 & 0.511 & 0.545 & 0.511\\
4 & Query & 0.5 & 16 & 5.10 & 3.65 & 0.557 & 0.615 & 0.511\\
4 & Query & 1.0 & 1 & 10.20 & 7.30 & 0.319 & 0.348 & 0.319\\
4 & Query & 1.0 & 4 & 10.20 & 7.30 & 0.319 & 0.353 & 0.319\\
4 & Query & 1.0 & 16 & 10.20 & 7.30 & 0.319 & 0.377 & 0.319\\
4 & Query & 2.0 & 1 & 20.40 & 14.60 & 0.273 & 0.302 & 0.273\\
4 & Query & 2.0 & 4 & 20.40 & 14.60 & 0.273 & 0.307 & 0.273\\
4 & Query & 2.0 & 16 & 20.40 & 14.60 & 0.273 & 0.331 & 0.273\\
\bottomrule
\end{tabular}
\end{table}

\subsection{Numerical validation requirements}
The conic feasibility check uses the tolerance and accounting in Algorithm~\ref{alg:capped-socp}. Primal cap and response residuals determine a feasible reconstructed upper endpoint; a lower endpoint requires a separately checked infeasibility witness. The intervals in Table~\ref{tab:gpt2-result-record} are not such witnesses.

For a zero-effect target, the ratio is undefined and absolute errors remain the appropriate report. For unequal numerical queries, the perturbation term is evaluated on the actual query discrepancies. Multihead compensation, earlier-layer prefix effects, shifted positions, and downstream token predictions remain outside this fixed-head comparison.

\end{document}